\documentclass[conference,10pt]{IEEEtran}
\IEEEoverridecommandlockouts
\usepackage{amsmath,amsfonts,amsthm,amssymb,bbm}
\usepackage{cite}
\usepackage{balance}
\usepackage{mathrsfs}
\usepackage{xcolor}
\usepackage{tikz}
\usepackage{pgfplots}
\usetikzlibrary{positioning,quotes,angles,patterns,intersections,pgfplots.fillbetween,calc}
\usepackage{tkz-euclide}
\usepackage{url}
\usetikzlibrary{decorations.pathmorphing}
\usepackage{enumitem}

\theoremstyle{plain} 
\newtheorem{theorem}{Theorem}

\newtheorem{definition}{Definition}

\theoremstyle{definition} \newtheorem{remark}{Remark}
\theoremstyle{definition} 
\usetikzlibrary{calc,shapes.geometric}

\title{Realizing GANs via a Tunable Loss Function}

\author{%
Gowtham R. Kurri, Tyler Sypherd, and Lalitha Sankar \\
 Arizona State University, \texttt{\{gkurri,tsypherd,lsankar\}@asu.edu}\thanks{This paper was presented in part at ITW 2021.} 
}
\begin{document}
\maketitle
\begin{abstract}
We introduce a tunable GAN, called $\alpha$-GAN, parameterized by $\alpha \in (0,\infty]$, which interpolates between various $f$-GANs and Integral Probability Metric based GANs (under constrained discriminator set).
We construct $\alpha$-GAN using a supervised loss function, namely, $\alpha$-loss, which is a tunable loss function capturing several canonical losses. 
We show that $\alpha$-GAN is intimately related to the Arimoto divergence, which was first proposed by \"{O}sterriecher (1996), and later studied by Liese and Vajda (2006). We also study the convergence properties of $\alpha$-GAN.
We posit that the holistic understanding that $\alpha$-GAN introduces will have practical benefits of addressing both the issues of vanishing gradients and mode collapse.
\end{abstract}

\section{Introduction}
In \cite{Goodfellow14}, Goodfellow~\emph{et al.} introduced \emph{generative adversarial networks} (GANs), a novel technique for training \emph{generative models} to produce samples from an unknown (true) distribution using a finite number of real samples. 
A GAN involves two learning models (both represented by deep neural networks in practice): a generator model $G$ that takes a random seed in a low-dimensional (relative to the data) \emph{latent} space to generate synthetic samples (by implicitly learning the true distribution without explicit probability models), and a discriminator model $D$ which classifies inputs (from either the true distribution or the generator) as real or fake. The generator wants to fool the discriminator while the discriminator wants to maximize the discrimination power between the true and generated samples. The opposing goals of $G$ and $D$ lead to a zero-sum min-max game in which a chosen value function is minimized and maximized over the model parameters of $G$ and $D$, respectively.


 For the value function considered in \emph{vanilla} GAN\footnote{We refer to the GAN introduced by Goodfellow \emph{et al.}~\cite{Goodfellow14} as \emph{vanilla} GAN, as done in the literature~\cite{lim2017geometric,cai2020utilizing} to distinguish it from others introduced later.}~\cite{Goodfellow14}, when $G$ and $D$ are given enough training time and capacity\footnote{Practical representations of $G$ and $D$, e.g., via neural networks, are able to capture a large class of functions when sufficiently parameterized.}, the min-max game is shown to have a Nash equilibrium leading to the generator minimizing the Jensen-Shannon divergence (JSD) between the true and the generated distributions. {Subsequently}, Nowozin \emph{et al.}\cite{NowozinCT16} showed that the GAN framework can minimize several $f$-divergences, including JSD, leading to $f$-GANs.
 Arguing that vanishing gradients are due to the sensitivity of $f$-divergences to mismatch in distribution supports, Arjovsky \emph{et al.}~\cite{ArjovskyCB17} proposed Wasserstein GAN (WGAN) using a ``weaker" Euclidean distance between distributions.
This has led to a broader class of GANs based on integral probability metric (IPM) distances~\cite{liang2018well}. Yet neither the vanilla GAN nor the IPM GANs perform consistently well in practice due to a variety of issues that arise during training (e.g., \emph{mode collapse, vanishing gradients, oscillatory convergence, to name a few})~\cite{huszar2015not,metz2016unrolled,salimans2016improved,arjovsky2017towards,GulrajaniAADC17}, thus providing even less clarity on how to choose the value function.

In this work, we first formalize a supervised loss function perspective of GANs and propose a tunable $\alpha$-GAN based on $\alpha$-loss, a class of tunable loss functions~\cite{sypherd2019tunable,SypherdDSD20} parameterized by $\alpha\in(0,\infty]$ that captures the well-known exponential loss ($\alpha=1/2$)~\cite{FREUND1997119}, the log-loss ($\alpha=1$)~\cite{MerhavF1998,CourtadeW11}, and the 0-1 loss ($\alpha=\infty$)~\cite{NguyenWJ09,BartlettJM06}.
Ultimately, we find that $\alpha$-GAN reveals a holistic structure in relating several canonical GANs, thereby unifying convergence and performance analyses. 
Our main contributions are as follows: 
\begin{itemize}
    \item We present a unique global Nash equilibrium to the min-max optimization problem induced by the $\alpha$-GAN, provided $G$ and $D$ have sufficiently large capacity and the models can be trained sufficiently long (Theorem~\ref{thm:alpha-GAN}). When the discriminator is trained to optimality (where its strategy under $\alpha$-loss is a tilted distribution), the generator seeks to minimize the \emph{Arimoto divergence} of order $\alpha$ (which has wide applications in statistics and information theory~\cite{LieseV06,osterreicher2003new}) between the true and the generated distributions, thereby providing an operational interpretation to the divergence. 
    We note that our approach differs from Nowozin \textit{et al.} $f$-GAN approach, please see Remark~\ref{remark1} for clarification.
    \item We show that $\alpha$-GAN interpolates between various $f$-GANs including vanilla GAN ($\alpha=1$), Hellinger GAN~\cite{NowozinCT16} ($\alpha=1/2$), Total Variation GAN~\cite{NowozinCT16} ($\alpha=\infty$), and IPM-based GANs including WGANs (when the discriminator set is appropriately constrained) by smoothly tuning the hyperparameter $\alpha$ (see Theorem~\ref{thm:fgans} and \eqref{eqn:IPM}). 
    Thus, $\alpha$-GAN allows a practitioner to determine how much they want to resemble vanilla GAN, for instance, since certain datasets/distributions may favor certain GANs (or even interpolation between certain GANs).
    Analogous to results on $\alpha$-loss in classification\cite{SypherdDSD20,sypherd2021journal}, where the model performance saturates quickly for $\alpha\rightarrow \infty$, we expect a similar saturation for $\alpha$-GAN (see Figure~\ref{fig:plotofdivergence}). Thus, we posit that smooth tuning from JSD to IPM that results from increasing $\alpha$ from $1$ to $\infty$ can address issues like mode collapse, vanishing gradients, etc. 
     \item In Theorem~\ref{prop}, we reconstruct the Arimoto divergence using the margin-based form of $\alpha$-loss~\cite{sypherd2021journal} and the variational formulation of Nguyen \textit{et al.}~\cite{NguyenWJ09}, which sheds more light on the convexity of the generator function of the divergence first proposed by \"{O}sterreicher~\cite{osterreicher1996class}, and later studied by and Liese and Vajda~\cite{LieseV06}.
    \item Finally, we study \emph{convergence} properties of $\alpha$-GAN in the presence of sufficiently large number of samples and discriminator capacity. We show that Arimoto divergences for all $\alpha>0$ are \emph{equivalent} in convergence (Theorem~\ref{thm:equivalenceinconvergence}) generalizing such an equivalence known in the literature~\cite{ArjovskyCB17,liu2017approximation} only for special cases, i.e., for $\alpha=1$ (Jensen-Shannon divergence), $\alpha=1/2$ (squared Hellinger distance), and $\alpha=\infty$ (total variation distance). 
\end{itemize}
The remainder of the paper is organized as follows. We review $\alpha$-loss and background on GANs in Section~\ref{section:preliminaries}. We present the loss function perspective of GANs in Section~\ref{section:lossfnperspective}. We propose and analyze tunable $\alpha$-GAN in Section~\ref{sectopn:tunableGAN}. Also, a connection between Arimoto divergence and the margin-based form of $\alpha$-loss is examined in Section~\ref{section:connections}. Finally, we study convergence properties of $\alpha$-GAN in Section~\ref{section:convergence}.
\section{$\alpha$-loss and GANs}\label{section:preliminaries}
We first review a tunable class of loss functions, $\alpha$-loss, that includes well-studied exponential loss ($\alpha=1/2$), log-loss ($\alpha=1$), and 0-1 loss ($\alpha=\infty$). Then, we present an overview of some related GANs in the literature. 
\begin{definition}[Sypherd \emph{et al.}~\cite{sypherd2021journal}]
\label{def:alphaloss} For a set of distributions $\mathcal{P}(\mathcal{Y})$ over $\mathcal{Y}$, $\alpha$-loss $\ell_{\alpha}:\mathcal{Y} \times \mathcal{P}(\mathcal{Y}) \rightarrow \mathbb{R}_{+}$ for $\alpha \in (0,1) \cup (1,\infty)$ is defined as
\begin{equation} \label{eq:alphaloss_prob}
\ell_{\alpha}(y,\hat{P}) \triangleq \frac{\alpha}{\alpha - 1}\left(1 - \hat{P}(y)^{\frac{\alpha-1}{\alpha}}\right).
\end{equation} 
By continuous extension, $\ell_{1}(y,\hat{P}) \triangleq -\log{\hat{P}(y)}$, $\ell_{\infty}(y,\hat{P}) \triangleq 1 - \hat{P}(y)$, and $\ell_{0}(y,\hat{P})\triangleq\infty$.
\end{definition}
Note that $\ell_{1/2}(y,\hat{P}) = \hat{P}(y)^{-1} - 1$, which is related to the exponential loss, particularly in the margin-based form~\cite{sypherd2021journal}. Also, $\alpha$-loss is convex in the probability term $\hat{P}(y)$.
Regarding the history of~\eqref{eq:alphaloss_prob}, Arimoto first studied $\alpha$-loss in finite-parameter estimation problems~\cite{arimoto1971information}, and later Liao \textit{et al.} used $\alpha$-loss to model the inferential capacity of an adversary to obtain private attributes~\cite{liao2018tunable}.
Most recently, Sypherd \textit{et al.} studied $\alpha$-loss in the machine learning setting~\cite{sypherd2021journal}, which is an impetus for this work.


\subsection{Background on GANs}

Let $P_r$ be a probability distribution over $\mathcal{X}\subset\mathbb{R}^d$, which the generator wants to learn \emph{implicitly} by producing samples by playing a competitive game with a discriminator in an adversarial manner. 
We parameterize the generator $G$ and the discriminator $D$ by vectors $\theta\in\Theta\subset \mathbb{R}^{n_g}$ and $\omega\in\Omega\subset\mathbb{R}^{n_d}$, respectively, and write $G_\theta$ and $D_\omega$ ($\theta$ and $\omega$ are typically the weights of neural network models for the generator and the discriminator, respectively). The generator $G_\theta$ takes as input a $d^\prime(\ll d)$-dimensional latent noise $Z\sim P_Z$ and maps it to a data point in $\mathcal{X}$ via the mapping $z\mapsto G_\theta(z)$. For an input $x\in\mathcal{X}$, the discriminator outputs $D_\omega(x)\in[0,1]$, the probability that $x$ comes from $P_r$ (real) as opposed to $P_{G_\theta}$ (synthetic). The generator and the discriminator play a two-player min-max game with a value function $V(\theta,\omega)$, resulting in a saddle-point optimization problem given by
\begin{align}\label{eqn:GANgeneral}
    \inf_{\theta\in\Theta}\sup_{\omega\in\Omega} V(\theta,\omega). 
\end{align}
Goodfellow \emph{et al.}~\cite{Goodfellow14} introduced a value function
\begin{align}
    &V_\text{VG}(\theta,\omega)\nonumber\\
    &=\mathbb{E}_{X\sim P_r}[\log{D_\omega(X)}]+\mathbb{E}_{Z\sim P_{Z}}[\log{(1-D_\omega(G_\theta(Z)))}]\nonumber\\
    &=\mathbb{E}_{X\sim P_r}[\log{D_\omega(X)}]+\mathbb{E}_{X\sim P_{G_\theta}}[\log{(1-D_\omega(X))}]\label{eq:Goodfellowobj}
\end{align}
and showed that when the discriminator class $\{D_\omega\}$, parametrized by $\omega$, is rich enough, \eqref{eqn:GANgeneral} simplifies to finding the $\inf_{\theta\in\Theta} 2D_{\text{JS}}(P_r||P_{G_\theta})-\log{4}$,
where $D_{\text{JS}}(P_r||P_{G_\theta})$ is the Jensen-Shannon divergence~\cite{Lin91} between $P_r$ and $P_{G_\theta}$. This simplification is achieved, for any $G_\theta$, by choosing the optimal discriminator
    \begin{align}
    D_{\omega^*}(x)=\frac{p_r(x)}{p_r(x)+p_{G_\theta}(x)},
    \end{align}
where $p_r$ and $p_{G_\theta}$ are the corresponding densities of the distributions $P_r$ and $P_{G_\theta}$, respectively, with respect to a base measure $dx$ (e.g., Lebesgue measure).

Generalizing this, Nowozin \emph{et al.}~\cite{NowozinCT16} derived value function
\begin{align}\label{eqn:fGANobj}
    V_f(\theta,\omega)=\mathbb{E}_{X\sim P_r}[D_\omega(X)]+\mathbb{E}_{X\sim P_{G_\theta}}[f^*(D_\omega(X))],
\end{align}
where\footnote{This is a slight abuse of notation in that $D_\omega$ is not a probability here. However, we chose this for consistency in  notation of discriminator across various GANs. 
} $D_\omega:\mathcal{X}\rightarrow \mathbb{R}$ and $f^*(t)\triangleq \sup_u\left\{ut-f(u)\right\}$ is the Fenchel conjugate of a convex lower semincontinuous function $f$,
for any $f$-divergence
$D_f(P_r||P_{G_\theta}):=\int_\mathcal{X}p_{G_\theta}(x)f\left(\frac{p_r(x)}{p_{G_\theta}(x)}\right)dx$~\cite{measures_renyi1961,Csiszar67,Alis66} (not just the Jensen-Shannon divergence) leveraging its variational characterization~\cite{NguyenWJ10}.
In particular, $\sup_{\omega\in\Omega} V_f(\theta,\omega)=D_f(P_r||P_{G_\theta})$ when there exists $\omega^*\in\Omega$ such that $T_{\omega^*}(x)=f^\prime\left(\frac{p_r(x)}{p_{G_\theta}(x)}\right)$. 

Highlighting the problems with the continuity of various $f$-divergences (e.g., Jensen-Shannon, KL, reverse KL, total variation) over the parameter space $\Theta$~\cite{arjovsky2017towards}, Arjovsky \emph{et al.}~\cite{ArjovskyCB17} proposed Wasserstein-GAN (WGAN) using the following Earth Mover's (also called Wasserstein-1) distance:
\begin{align}
    W(P_r,P_{G_\theta})
    =\inf_{\Gamma_{X_1X_2}\in\Pi(P_r,P_{G_\theta})}\mathbb{E}_{(X_1,X_2)\sim \Gamma_{X_1X_2}}\lVert{X_1-X_2}\rVert_2,  
\end{align}
where $\Pi(P_r,P_{G_\theta})$ is the set of all joint distributions $\Gamma_{X_1X_2}$ with marginals $P_r$ and $P_{G_\theta}$. WGAN employs the Kantorovich-Rubinstein duality \cite{villani2008optimal} using the value function
\begin{align}\label{eqn:WGANobj}
   V_\text{WGAN}(\theta,\omega)=\mathbb{E}_{X\sim P_r}[D_\omega(X)]-\mathbb{E}_{X\sim P_{G_\theta}}[D_\omega(X)],
\end{align}
where the functions $D_\omega:\mathcal{X}\rightarrow \mathbb{R}$ are all 1-Lipschitz,
to simplify $\sup_{\omega\in\Omega}V_{\text{WGAN}}(\theta,\omega)$ to $W(P_r,P_{G_\theta})$ when the class $\Omega$ is rich enough. Although, various GANs have been proposed in the literature, each of them exhibits their own strengths and weaknesses in terms of convergence, vanishing gradients, mode collapse, computational complexity, etc. leaving the problem of instability unsolved~\cite{wiatrak2019stabilizing}.

\section{Loss Function Pespective of GANs}\label{section:lossfnperspective}
Noting that a GAN involves a classifier (i.e., discriminator), it is well known that the value function $V_{\text{VG}}(\theta,\omega)$ in \eqref{eq:Goodfellowobj} considered by Goodfellow \emph{et al.}~\cite{Goodfellow14} is related to cross-entropy loss. While perhaps it has not been explicitly articulated heretofore in the literature, we first formalize this loss function perspective of GANs. In \cite{AroraGLMZ17}, Arora \emph{et al.}~observed that the $\log$ function in \eqref{eq:Goodfellowobj} can be replaced by any (monotonically increasing) concave function $\phi(x)$ (e.g., $\phi(x)=x$ for WGANs). More generally, we show that one can write $V(\theta,\omega)$ in terms of \emph{any} classification loss $\ell(y,\hat{y})$ with inputs $y\in\{0,1\}$ (the true label) and $\hat{y}\in[0,1]$ (soft prediction of $y$). For a GAN, we have $(X|y=1)\sim P_r$, $(X|y=0)\sim P_{G_\theta}$, and $\hat{y}=D_\omega(x)$. With this, we define a value function 
\begin{align}
    V(\theta,\omega)&=\mathbb{E}_{X|y=1}[-\ell(y,D_\omega(X))]+\mathbb{E}_{X|y=0}[-\ell(y,D_\omega(X))]\\
    &=\mathbb{E}_{X\sim P_r}[-\ell(1,D_\omega(X))]+\mathbb{E}_{X\sim P_{G_\theta}}[-\ell(0,D_\omega(X))]\label{eqn:lossfnps1}.
\end{align}
For cross-entropy loss, i.e., $\ell_{\text{CE}}(y,\hat{y})\triangleq -y\log{\hat{y}}-(1-y)\log{(1-\hat{y})}$, notice that the expression in \eqref{eqn:lossfnps1} is equal to $V_{\text{VG}}$ in \eqref{eq:Goodfellowobj}.
For the value function in \eqref{eqn:lossfnps1}, we consider a GAN given by the min-max optimization problem:
\begin{align}\label{eqn:lossfnbasedGAN}
    \inf_{\theta\in\Theta}\sup_{\omega\in\Omega}V(\theta,\omega).
\end{align}
Let $\phi(\cdot):=-\ell(1,\cdot)$ and $\psi(\cdot):=-\ell(0,\cdot)$ in the sequel. The functions $\phi$ and $\psi$ are assumed to be monotonically increasing and decreasing functions, respectively, so as to retain the intuitive interpretation of the vanilla GAN (that the discriminator should output high values to real samples and low values to the generated samples). These functions should also satisfy the constraint
\begin{align}\label{eqn:condnonfnsforGAN}
    \phi(t)+\psi(t)\leq \phi({1}/{2})+\psi({1}/{2}),\ \text{for all}\ t\in[0,1], 
\end{align}
so that the optimal discriminator guesses uniformly at random (i.e., outputs a constant value ${1}/{2}$ irrespective of the input) when $P_r=P_{G_\theta}$. A loss function $\ell(y,\hat{y})$ is said to be \emph{symmetric}~\cite{reid2010composite} if $\psi(t)=\phi(1-t)$, for all $t\in[0,1]$. Notice that the GAN considered by Arora \emph{et al.}~\cite[(2)]{AroraGLMZ17} is a specials case of $\eqref{eqn:lossfnbasedGAN}$, in particular, $\eqref{eqn:lossfnbasedGAN}$ recovers the form of GAN in \cite[(2)]{AroraGLMZ17} when the loss function $\ell(y,\hat{y})$ is symmetric. For symmetric losses, concavity of the function $\phi$ is a sufficient condition for satisfying \eqref{eqn:condnonfnsforGAN}, but not a necessary condition.
\section{Tunable $\alpha$-GAN}\label{sectopn:tunableGAN}
In this section, we examine the loss function perspective of GANs by focusing on the GAN obtained by plugging in $\alpha$-loss. We first write $\alpha$-loss in \eqref{eq:alphaloss_prob} in the form of a binary classification loss to obtain
\begin{align}
    \ell_\alpha(y,\hat{y}):=\frac{\alpha}{\alpha-1}\left(1-y\hat{y}^{\frac{\alpha-1}{\alpha}}-(1-y)(1-\hat{y})^{\frac{\alpha-1}{\alpha}}\right),\label{eqn:alphaloss}
\end{align}
 for $\alpha\in(0,1)\cup (1,\infty)$. Note that \eqref{eqn:alphaloss} recovers $\ell_{\text{CE}}$ as $\alpha\rightarrow 1$. Now consider a \emph{tunable $\alpha$-GAN} with a value function 
\begin{align}
    &V_\alpha(\theta,\omega)\nonumber\\
    &=\mathbb{E}_{X\sim P_r}[-\ell_{\alpha}(1,D_\omega(X))]+\mathbb{E}_{X\sim P_{G_\theta}}[-\ell_{\alpha}(0,D_\omega(X))]\nonumber\\
    &=\frac{\alpha}{\alpha-1}\times\nonumber\\
    &\left(\mathbb{E}_{X\sim P_r}\left[D_\omega(X)^{\frac{\alpha-1}{\alpha}}\right]+\mathbb{E}_{X\sim P_{G_\theta}}\left[\left(1-D_\omega(X)\right)^{\frac{\alpha-1}{\alpha}}\right]-2\right)\label{eqn:alphaGANobjective}.
\end{align}
We can verify that $\lim_{\alpha\rightarrow 1}V_\alpha(\theta,\omega)=V_{\text{VG}}(\theta,\omega)$ recovering the value function of the vanilla GAN. Also, notice that 
\begin{align}\label{eqn:IPM}
\lim_{\alpha\rightarrow \infty}V_\alpha(\theta,\omega)=\mathbb{E}_{X\sim P_r}\left[D_\omega(x)\right]-\mathbb{E}_{X\sim P_{G_\theta}}\left[D_\omega(x)\right]-1
\end{align}
is the value function (modulo a constant) used in Intergral Probability Metric (IPM) based GANs\footnote{Note that IPMs do not restrict the function $D_\omega$ to be a probability.}, e.g., WGAN, McGan~\cite{pmlr-v70-mroueh17a}, Fisher GAN~\cite{MrouehS17}, and Sobolev GAN~\cite{mroueh2017sobolev}.
The resulting min-max game in $\alpha$-GAN is given by
\begin{align} 
\inf_{\theta\in\Theta}\sup_{\omega\in\Omega}V_\alpha(\theta,\omega)\label{eqn:minimaxalphaGAN}.
\end{align}
The following theorem provides the min-max solution, i.e., Nash equilibrium, to the two-player game in \eqref{eqn:minimaxalphaGAN} for the non-parametric setting, i.e., when the discriminator set $\Omega$ is large enough.
\begin{theorem}[min-max solution]\label{thm:alpha-GAN}
For a fixed generator $G_\theta$, the discriminator $D_{\omega^*}(x)$ optimizing the $\sup$ in \eqref{eqn:minimaxalphaGAN} is given by
\begin{align}\label{eqn:optimaldoisc}
    D_{\omega^*}(x)=\frac{p_r(x)^\alpha}{p_r(x)^\alpha+p_{G_\theta}(x)^\alpha}.
\end{align}
For this $D_{\omega^*}(x)$, \eqref{eqn:minimaxalphaGAN} simplifies to minimizing a non-negative symmetric $f_\alpha$-divergence $D_{f_\alpha}(\cdot||\cdot)$ as
\begin{align}\label{eqn:inf-obj-alpha}
    \inf_{\theta\in\Theta} D_{f_\alpha}(P_r||P_{G_\theta})+\frac{\alpha}{\alpha-1}\left(2^{\frac{1}{\alpha}}-2\right),
\end{align}
where
\begin{align}\label{eqn:falpha}
f_\alpha(u)=\frac{\alpha}{\alpha-1}\left(\left(1+u^\alpha\right)^{\frac{1}{\alpha}}-(1+u)-2^{\frac{1}{\alpha}}+2\right),
\end{align}
for $u\geq 0$ and\footnote{We note that the divergence $D_{f_\alpha}$ has been referred to as \emph{Arimoto divergence} in the literature~\cite{osterreicher1996class,osterreicher2003new,LieseV06}. We refer the reader to Section~\ref{section:connections} for more details.}
\begin{align}\label{eqn:alpha-divergence}
D_{f_\alpha}(P||Q)=\frac{\alpha}{\alpha-1}\left(\int_\mathcal{X} \left(p(x)^\alpha+q(x)^\alpha\right)^\frac{1}{\alpha} dx-2^{\frac{1}{\alpha}}\right),
\end{align}
which is minimized iff $P_{G_\theta}=P_r$. 
\end{theorem}

\begin{figure}[htbp] 
    \centerline{\includegraphics[scale=.4]{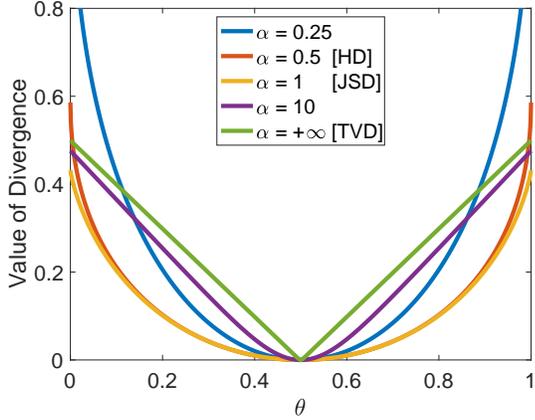}}
    \caption{A plot of $D_{f_{\alpha}}$ in~\eqref{eqn:alpha-divergence} for several values of $\alpha$ where $p \sim \text{Ber}(1/2)$ and $q \sim \text{Ber}(\theta)$. Note that HD, JSD, and TVD, are abbreviations for Hellinger, Jensen-Shannon, and Total Variation divergences, respectively. As $\alpha \rightarrow 0$, the curvature of the divergence increases, placing increasingly more weight on $\theta \neq 1/2$. 
    Conversely, for $\alpha \rightarrow \infty$, $D_{f_{\alpha}}$ quickly resembles $D_{f_{\infty}}$, hence a saturation effect of $D_{f_{\alpha}}$. 
    }
    \label{fig:plotofdivergence}
\end{figure}

\begin{remark}\label{remark1}
It can be inferred from \eqref{eqn:inf-obj-alpha} that when the discriminator is trained to optimality, the generator has to minimize the $f_\alpha$-divergence hinting at an application of $f$-GAN instead. Implementing $f_\alpha$-GAN directly via value function in~\eqref{eqn:fGANobj} (for $f_\alpha$) involves finding convex conjugate of $f_\alpha$, which is challenging in terms of computational complexity making it inconvenient for optimization in the training phase of GANs. In contrast, our approach of using supervised losses circumvents this tedious effort \emph{and} also provides an operational interpretation of $f_\alpha$-divergence via losses.
A related work where an $f$-divergence (in particular, $\alpha$-divergence~\cite{Amari1985}) shows up in the context of GANs, even when the problem formulation is not via $f$-GAN, is by Cai \emph{et al.}~\cite{cai2020utilizing}. 
However, our work differs from~\cite{cai2020utilizing} in that the value function we use is well motivated via supervised loss functions of binary classification and also recovers the basic GAN~\cite{Goodfellow14} (among others). 
\end{remark}

\begin{remark}
As $\alpha \rightarrow 0$, note that~\eqref{eqn:optimaldoisc} implies a more cautious discriminator, i.e., if $p_{G_{\theta}}(x) \geq p_{r}(x)$, then $D_{w^{*}}(x)$ decays more slowly from $1/2$, and if $p_{G_{\theta}}(x) \leq p_{r}(x)$, $D_{w^{*}}(x)$ increases more slowly from $1/2$. 
Conversely, as $\alpha\rightarrow\infty$,~\eqref{eqn:optimaldoisc} simplifies to $D_{\omega^*}(x)=\mathbbm{1}\{p_r(x)>p_{G_\theta}(x)\}+\frac{1}{2}\mathbbm{1}\{p_r(x)=p_{G_\theta}(x)\}$, where the discriminator implements the Maximum Likelihood (ML) decision rule, i.e., a hard decision whenever $p_r(x)\neq p_{G_\theta}(x)$. In other words,~\eqref{eqn:optimaldoisc} for $\alpha \rightarrow \infty$ induces a very confident discriminator.
Regarding the generator's perspective,~\eqref{eqn:inf-obj-alpha} (and Figure~\ref{fig:plotofdivergence}) implies that the generator seeks to minimize the discrepancy between $P_{r}$ and $P_{G_{\theta}}$
according to the geometry induced by $D_{f_\alpha}$.
Thus, the optimization trajectory traversed by the generator during training is strongly dependent on the practitioner's choice of $\alpha \in (0,\infty]$. 
Please refer to Figure~\ref{fig:simplexillustration} for an illustration of this observation. 
%
\end{remark}
\begin{figure}[h]
\centering




\begin{tikzpicture}
 \node[draw,very thick,regular polygon,regular polygon sides=3,minimum width=7cm](ternary){};
\draw[blue, thick,->] (-.75,1) arc (96:16:2.75) node[rotate=-37] at (5:.875) {\scriptsize \textcolor{blue}{$\inf\limits_{\theta\in\Theta} D_{f_{\alpha_{1}}}(P_{r}||P_{G_{\theta}})$}};
\draw[red, thick,->] (-.75,1) arc (190:270:2.72) node[rotate=-39] at (95:-.87) {\scriptsize \textcolor{red}{$\inf\limits_{\theta\in\Theta} D_{f_{\alpha_{2}}}(P_{r}||P_{G_{\theta}})$}};
\filldraw [black] (2.2,-1.25) circle (2pt) node[anchor=north west] {$P_{r}$};
\filldraw [black] (-.75,1) circle (2pt) node[anchor=south] {$P_{G_{\theta}}$};
\end{tikzpicture}
\caption{An idealized illustration on the probability simplex of the infimum over $\theta$ in~\eqref{eqn:inf-obj-alpha} for $\alpha_{1},\alpha_{2} \in (0,\infty]$ such that $\alpha_{1} \neq \alpha_{2}$. The choice of $\alpha$ in the min-max game for the $\alpha$-GAN in~\eqref{eqn:minimaxalphaGAN} defines the optimization trajectory taken by the generator (versus an optimal discriminator as specified in~\eqref{eqn:optimaldoisc}) by distorting the underlying geometry according to $D_{f_{\alpha}}$. 
}
    \label{fig:simplexillustration}
\end{figure}
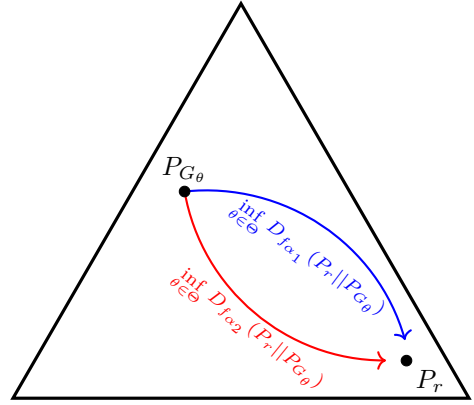
A detailed proof of Theorem~\ref{thm:alpha-GAN} is in Appendix~\ref{proofofthm1}. For intuition on the construction of the function in \eqref{eqn:falpha}, see Theorem~\ref{prop}. Next we show that $\alpha$-GAN recovers various well known $f$-GANs.
\begin{theorem}[$f$-GANs]\label{thm:fgans}
$\alpha$-GAN recovers vanilla GAN, Hellinger GAN (H-GAN)~\cite{NowozinCT16}, and Total Variation GAN (TV-GAN)~\cite{NowozinCT16} as $\alpha\rightarrow 1$, $\alpha=\frac{1}{2}$, and $\alpha\rightarrow \infty$, respectively.
\end{theorem}
A detailed proof is in Appendix~\ref{proofoftheorem2}.
\subsection{Reconstructing Arimoto Divergence}\label{section:connections}
%
It is interesting to note that the divergence $D_{f_\alpha}(\cdot||\cdot)$ (in \eqref{eqn:alpha-divergence}) that naturally emerges from the analysis of $\alpha$-GAN was first proposed by \"{O}sterriecher~\cite{osterreicher1996class} in the context of statistics and was later referred to as the \emph{Arimoto divergence} by Liese and Vajda~\cite{LieseV06}. It was shown to have several desirable properties with applications in statistics and information theory~\cite{cerone2004bound,vajda2009metric}. For example:
\begin{itemize}
\item A geometric interpretation of the divergence $D_{f_\alpha}$ in the context of hypothesis testing~\cite{osterreicher1996class}.
\item $D_{f_\alpha}(P||Q)^{\min\{\alpha,\frac{1}{2}\}}$ defines a distance metric (satisfying the triangle inequality) on the set of probability distributions~\cite{osterreicher2003new}. 
\end{itemize}
When the Arimoto divergence $D_{f_\alpha}$ was proposed, the convexity of the generating function $f_\alpha$ was proved via the traditional second derivative test~\cite[Lemma~1]{osterreicher1996class}.
We present an alternative approach to arriving at the Arimoto divergence by utilizing the margin-based\footnote{In the binary classification context, the margin is represented by $t := yf(x)$, where $x \in \mathcal{X}$ is the feature vector, $y \in \{-1,+1\}$ is the label, and $f: \mathcal{X} \rightarrow \mathbb{R}$ is the prediction function produced by a learning algorithm.} form of $\alpha$-loss (see~\cite{sypherd2021journal}) where the convexity of $f_\alpha$ (and also the symmetric property of $D_{f_\alpha}(\cdot||\cdot)$) arises in a rather natural manner, thereby reconstructing the Arimoto divergence through a distinct conceptual perspective. 

We do this by noticing that the Arimoto divergence falls into the category of a broad class of $f$-divergences that can be obtained from margin-based loss functions. {Such a connection between margin-based losses in classification and the corresponding $f$-divergences was introduced} 
by Nguyen \emph{et al.}~\cite[Theorem~1]{NguyenWJ09}. They observed that, for a given margin-based loss function $\tilde{\ell}$, there is a corresponding $f$-divergence with the convex function $f$ defined as $f(u):=-\inf_t\left(u\tilde{\ell}(t)+\tilde{\ell}(-t)\right)$. The convexity of $f$ follows simply because the infimum of affine functions is concave, and this argument does not require $\tilde{\ell}$ to be convex\footnote{in fact $\alpha$-loss in its margin-based form is only quasi-convex for $\alpha>1$}. Additionally, the $f$-divergence obtained is always symmetric because $f$ satisfies $f(u)=uf(\frac{1}{u})$ since
   $\inf_tu\tilde{\ell}(t)+\tilde{\ell}(-t)= \inf_t\tilde{\ell}(t)+u\tilde{\ell}(-t)$. 

The margin-based $\alpha$-loss~\cite{sypherd2019tunable} for $\alpha\in(0,1)\cup (1,\infty)$, $\tilde{\ell}_\alpha:\bar{\mathbb{R}}\rightarrow \mathbb{R}_+$ is defined as
\begin{align}
    \tilde{\ell}_\alpha(t)\triangleq \frac{\alpha}{\alpha-1}\left(1-{\sigma(t)}^{\frac{\alpha-1}{\alpha}}\right),
\end{align}
where $\sigma:\bar{\mathbb{R}}\rightarrow \mathbb{R}_+$ is the sigmoid function given by $\sigma(t)=(1+\mathrm{e}^{-t})^{-1}$. With these preliminaries in hand, we have the following result.
\begin{theorem}\label{prop}
For the function $f_\alpha$ in \eqref{eqn:falpha}, it holds that
\begin{multline}\label{eqn:thm3stat}
    f_\alpha(u)=-\inf_t\left(u\tilde{\ell}_\alpha(t)+\tilde{\ell}_\alpha(-t)\right)-\frac{\alpha}{\alpha-1}\left(2^{\frac{1}{\alpha}}-2\right)\\
  \text{for}\ u\geq 0.
\end{multline}
\end{theorem}
A detailed proof is in Appendix~\ref{proofoftheorem6}.

\section{Convergence Properties of $\alpha$-GAN}\label{section:convergence}
We study \emph{convergence} properties of $\alpha$-GAN under the assumption of sufficiently large number of samples and discriminator capacity (similar to \cite{liu2017approximation}). Liu~\emph{et al.}~\cite{liu2017approximation} addressed a fundamental question in the context of convergence analysis of GANs, in general: For a sequence of generated distributions, does convergence of a divergence between the generated distribution and a fixed real distribution (that the generator wants to minimize) to the global minimum lead to some standard notion of distributional convergence to the real distribution? They answer this question in the affirmative provided $\mathcal{X}$ is a compact metric space. They define \emph{adversarial divergences}~\cite[Definition~1]{liu2017approximation} which capture the divergences used by a number of existing GANs that include vanilla GAN~\cite{Goodfellow14}, $f$-GAN~\cite{NowozinCT16}, WGAN~\cite{ArjovskyCB17}, and MMD-GAN~\cite{dziugaite2015training}. For \emph{strict} adversarial divergences (a subclass of the adversarial divergences where the minimizer of the divergence is uniquely the real distribution), they showed that convergence of divergence to global minimum implies weak convergence of the generated distribution to the real distribution. They also obtain a structural result on the class of strict adversarial divergences~\cite[Corollary~12]{liu2017approximation} based on a notion of \emph{relative strength} between adversarial divergences.  

We first borrow the following terminology from Liu \emph{et al.}~\cite{liu2017approximation} in order to study convergence properties of $\alpha$-GAN: A \emph{strict adversarial divergence} $\tau_1$ is said to be stronger than another adversarial divergence $\tau_2$ (or $\tau_2$ is said to be weaker than $\tau_1$) if for any sequence of probability distributions $(P_n)$ and target distribution $P$, $\tau_1(P\|P_n)\rightarrow 0$ implies $\tau_2(P\|P_n)\rightarrow 0$. We say $\tau_1$ is equivalent to $\tau_2$ is $\tau_1$ is both stronger and weaker than $\tau_2$. We say $\tau_1$ is strictly stronger than $\tau_2$ if $\tau_1$ is stronger than $\tau_2$ but not equivalent. We say $\tau_1$ and $\tau_2$ are not comparable if $\tau_1$ is neither stronger nor weaker than $\tau_2$.

Arjovsky \emph{et al.}~\cite{ArjovskyCB17} proved that Jensen-Shannon divergence is equivalent to total variation distance. Later Liu \emph{et al.} showed that squared Hellinger distance is also equivalent to both these divergences, meaning that all the three divergences belong to the same equivalence class (see \cite[Figure~1]{liu2017approximation}). Noticing that squared Hellinger distance, Jensen-Shannon divergence, and total variation distance correspond to Arimoto divergences $D_{f_\alpha}(\cdot||\cdot)$ for $\alpha=1/2$, $\alpha=1$, and $\alpha=\infty$, respectively, it is natural to ask the question: Are Arimoto divergences for all $\alpha>0$ equivalent? We answer this question in the affirmative in Theorem~\ref{thm:equivalenceinconvergence}, thereby adding the Arimoto divergences for all other $\alpha$ also to the same equivalence class. As a motivation for a proof of this, we first give an alternative and more simpler proof for the equivalence of Jensen-Shannon divergence and total variation distance ~\cite[Theorem 2(1)]{ArjovskyCB17}.
\begin{theorem}
The Jensen-Shannon divergence is equivalent to the total variation distance.
\end{theorem}
\begin{proof}
We first show that the total variation distance is stronger than the Jensen-Shannon divergence. Consider a sequence of probability distributions $(P_n)$ such that $D_{\text{TV}}(P_n||P)\rightarrow 0$. Using the fact that the total variation distance upper bounds the Jensen-Shannon divergence~\cite[Theorem 3]{Lin91}, we have $D_{\text{JS}}(P_n||P)\leq (\log{2})D_{\text{TV}}(P_n||P)$, for each $n\in\mathbbm{N}$. This implies that $D_{\text{JS}}(P_n||P)\rightarrow 0$ since $D_{\text{TV}}(P_n||P)\rightarrow 0$. This greatly simplifies the corresponding proof of \cite[Theorem 2(1)]{ArjovskyCB17} which uses measure-theoretic analysis, in particular, the Radon-Nikodym theorem. The proof for the other direction, i.e., the Jensen-Shannon divergence is stronger than the total variation distance, is exactly along the same lines as that of \cite[Theorem~2(1)]{ArjovskyCB17} using triangular and Pinsker's inequalities.   
\end{proof}

\begin{theorem}\label{thm:equivalenceinconvergence}
Arimoto divergences for all $\alpha>0$ are equivalent. That is, for a sequence of probability distributions $(P_n)$, $D_{f_{\alpha_1}}(P_n||P)\rightarrow 0$ if and only if $D_{f_{\alpha_2}}(P_n||P)\rightarrow 0$ for any $\alpha_1\neq \alpha_2$. 
\end{theorem}
A detailed proof is in Appendix~\ref{proofoftheorem4}.

\section{Conclusion}
We have shown that a classical information-theoretic measure (Arimoto divergence) characterizes the ideal performance of a modern machine learning algorithm ($\alpha$-GAN) which interpolates between several canonical GANs.
For future work, we will investigate $\alpha$-GAN in practice, with particular interest in its \emph{generalization} guarantees and its efficacy to reduce \textit{mode collapse}.

\appendices
\section{Proof of Theorem~\ref{thm:alpha-GAN}}\label{proofofthm1}
For a fixed generator, $G_\theta$, we first solve the optimization problem
\begin{align}
   \sup_{\omega\in\Omega}\int_\mathcal{X}\frac{\alpha}{\alpha-1}\left(p_r(x)D_\omega(x)^{\frac{\alpha-1}{\alpha}}+p_{G_\theta}(x)(1-D_\omega(x))^{\frac{\alpha-1}{\alpha}}\right).
\end{align}
Consider the function
\begin{align}
    g(y)=\frac{\alpha}{\alpha-1}\left(ay^{\frac{\alpha-1}{\alpha}}+b(1-y)^{\frac{\alpha-1}{\alpha}}\right),
\end{align}
for $a,b\in\mathbb{R}_+$ and $y\in[0,1]$. To show that the optimal discriminator is given by the expression in \eqref{eqn:optimaldoisc}, it suffices to show that $g(y)$ achieves its maximum in $[0,1]$ at $y^*=\frac{a^\alpha}{a^\alpha+b^\alpha}$. Notice that for $\alpha>1$, $y^{\frac{\alpha-1}{\alpha}}$ is a concave function of $y$, meaning the function $g$ is concave. For $0<\alpha<1$, $y^{\frac{\alpha-1}{\alpha}}$ is a convex function of $y$, but since $\frac{\alpha}{\alpha-1}$ is negative, the overall function $g$ is again concave. Consider the derivative 
    $g^\prime(y^*)=0$,
which gives us
\begin{align}
    y^*=\frac{a^\alpha}{a^\alpha+b^\alpha}.
\end{align}
This gives \eqref{eqn:optimaldoisc}. With this, the optimization problem in \eqref{eqn:minimaxalphaGAN} can be written as $\inf_{\theta\in\Theta}C(G_\theta)$,
where
\begin{align}
   &C(G_\theta)=\frac{\alpha}{\alpha-1}\times\nonumber\\
   &\left[\int_\mathcal{X}\left(p_r(x)D_{\omega^*}(x)^{\frac{\alpha-1}{\alpha}}+p_{G_\theta}(x)(1-D_{\omega^*}(x))^{\frac{\alpha-1}{\alpha}}\right)dx-2\right]\\
   &=\frac{\alpha}{\alpha-1}\Bigg[\int_\mathcal{X}\Bigg(p_r(x)\left( \frac{p_r(x)^\alpha}{p_r(x)^\alpha+p_{G_\theta}(x)^\alpha}\right)^{\frac{\alpha-1}{\alpha}}+\nonumber\\
   &\hspace{12pt}p_{G_\theta}(x)\left( \frac{p_r(x)^\alpha}{p_r(x)^\alpha+p_{G_\theta}(x)^\alpha}\right)^{\frac{\alpha-1}{\alpha}}\Bigg)dx-2\Bigg]\\
    &=\frac{\alpha}{\alpha-1}\left(\int_{\mathcal{X}}\left(p_r(x)^\alpha+p_{G_\theta}(x)^\alpha\right)^{\frac{1}{\alpha}}dx-2\right)\\
    &=D_{f_\alpha}(P_r||P_{G_\theta})+\frac{\alpha}{\alpha-1}\left(2^{\frac{1}{\alpha}}-2\right),
\end{align}
where for the convex function $f_\alpha$ in \eqref{eqn:falpha},
\begin{align}
    &D_{f_\alpha}(P_r||P_{G_\theta})=\int_\mathcal{X} p_{G_\theta}(x)f_\alpha\left(\frac{p_r(x)}{p_{G_\theta}(x)}\right) dx\\
    &=\frac{\alpha}{\alpha-1}\left(\int_{\mathcal{X}}\left(p_r(x)^\alpha+p_{G_\theta}(x)^\alpha\right)^{\frac{1}{\alpha}}dx-2^{\frac{1}{\alpha}}\right).
\end{align}
This gives us \eqref{eqn:inf-obj-alpha}. Since $D_{f_\alpha}(P_r||P_{G_\theta})\geq 0$ with equality if and only if $P_r=P_{G_\theta}$, we have $C(G_\theta)\geq \frac{\alpha}{\alpha-1}\left(2^{\frac{1}{\alpha}}-2\right)$ with equality if and only if $P_r=P_{G_\theta}$.
\balance
\section{Proof of Theorem~\ref{thm:fgans}}\label{proofoftheorem2}
First, using L'H\^{o}pital's rule we can verify that, for $a,b>0$,
\begin{multline}
\lim_{\alpha\rightarrow 1}\frac{\alpha}{\alpha-1}\left(\left(a^\alpha+b^\alpha\right)^{\frac{1}{\alpha}}-2^{\frac{1}{\alpha}-1}(a+b)\right)\\
=a\log{\left(\frac{a}{\frac{a+b}{2}}\right)}+b\log{\left(\frac{b}{\frac{a+b}{2}}\right)}.
\end{multline}
Using this, we have
\begin{align}
&D_{f_1}(P_r||P_{G_\theta})\triangleq\lim_{\alpha\rightarrow 1}D_{f_\alpha}(P_r||P_{G_\theta})\\
&=\lim_{\alpha\rightarrow 1}\frac{\alpha}{\alpha-1}\left(\int_\mathcal{X}\left(p_r(x)^\alpha+p_{G_\theta}(x)^\alpha\right)^{\frac{1}{\alpha}}dx-2^{\frac{1}{\alpha}}\right)\\
&=\lim_{\alpha\rightarrow 1}\Bigg[\frac{\alpha}{\alpha-1}\times\nonumber\\
&\int_\mathcal{X}\left(\left(p_r(x)^\alpha+p_{G_\theta}(x)^\alpha\right)^{\frac{1}{\alpha}}-2^{\frac{1}{\alpha}-1}(p_r(x)+p_{G_\theta}(x))\right)dx\Bigg]\\
&=\int_{\mathcal{X}}p_r(x)\log{\frac{p_r(x)}{\left(\frac{p_r(x)+p_{G_\theta}(x)}{2}\right)}}dx+\nonumber\\
&\hspace{12pt}\int_{\mathcal{X}}p_{G_\theta}(x)\log{\frac{p_{G_\theta}(x)}{\left(\frac{p_r(x)+p_{G_\theta}(x)}{2}\right)}}dx\\
&=:2D_{\text{JS}}(P_r||P_{G_\theta}),
\end{align}
where $D_{\text{JS}}(\cdot||\cdot)$ is the Jensen-Shannon divergence.
Now, as $\alpha\rightarrow 1$, \eqref{eqn:inf-obj-alpha} equals $\inf_{\theta\in\Theta}2D_{\text{JS}}(P_r||P_{G_\theta})-\log{4}$ recovering vanilla GAN.

Substituting $\alpha=\frac{1}{2}$ in \eqref{eqn:alpha-divergence}, we get
\begin{align}
    D_{f_{\frac{1}{2}}}(P_r||P_{G_\theta})&=-\int_\mathcal{X}\left(\sqrt{p_r(x)}+\sqrt{p_{G_\theta}(x)}\right)^2dx+4\\
    &=\int_{\mathcal{X}}\left(\sqrt{p_r(x)}-\sqrt{p_{G_\theta}(x)}\right)^2dx\\
    &=:2D_{\text{H}^2}(P_r||P_{G_\theta}),
\end{align}
where $D_{\text{H}^2}(P_r||P_{G_\theta})$ is the squared Hellinger distance. For $\alpha=\frac{1}{2}$, \eqref{eqn:inf-obj-alpha} gives $2\inf_{\theta\in\Theta}D_{\text{H}^2}(P_r||P_{G_\theta})-2$ recovering Hellinger GAN (up to a constant). 

Noticing that, for $a,b>0$, $\lim_{\alpha\rightarrow \infty}\left(a^\alpha+b^\alpha\right)^{\frac{1}{\alpha}}=\max\{a,b\}$ and defining $\mathcal{A}:=\{x\in\mathcal{X}:p_r(x)\geq p_{G_\theta}(x)\}$, we have
\begin{align}
&D_{f_1}(P_r||P_{G_\theta})\triangleq \lim_{\alpha\rightarrow\infty}D_{f_\alpha}(P_r||P_{G_\theta})\\
&=\lim_{\alpha\rightarrow\infty}\frac{\alpha}{\alpha-1}\left(\int_\mathcal{X}\left(p_r(x)^\alpha+p_{G_\theta}(x)^\alpha\right)^{\frac{1}{\alpha}}dx-2^{\frac{1}{\alpha}}\right)\\
&=\int_\mathcal{X}\max\{p_r(x),p_{G_\theta}(x)\}\ dx-1\\
&=\int_{\mathcal{X}}\max\{p_r(x)-p_{G_\theta}(x),0\}\ dx\\
&=\int_{\mathcal{A}}(p_r(x)-p_{G_\theta}(x))\ dx\\
&=\int_{\mathcal{A}}\frac{p_r(x)-p_{G_\theta}(x)}{2}\ dx+\int_{\mathcal{A}^c}\frac{p_{G_\theta}(x)-p_r(x)}{2}\ dx\\
&=\frac{1}{2}\int_{\mathcal{X}}\left|p_r(x)-p_{G_\theta}(x)\right|\ dx\\
&=:D_{\text{TV}}(P_r||P_{G_\theta}),
\end{align}
where $D_{\text{TV}}(P_r||P_{G_\theta})$ is the total variation distance between $P_r$ and $P_{G_\theta}$. Thus, as $\alpha\rightarrow\infty$, \eqref{eqn:inf-obj-alpha} equals $\inf_{\theta\in\Theta}D_{\text{TV}}(P_r||P_{G_\theta})-1$ recovering TV-GAN (modulo a constant).
\section{Proof of Theorem~\ref{thm:equivalenceinconvergence}}\label{proofoftheorem4}
Consider a sequence of probability distributions $(P_n)$. To prove the theorem, notice that it suffices to show that Arimoto divergence for any $\alpha>0$ is equivalent to the total variation distance, i.e., $D_{f_\alpha}(P_n||P)\rightarrow 0$ if and only if $D_{\text{TV}}(P_n||P)\rightarrow 0$. To this end, we employ a property of Arimoto divergence which gives lower and upper bounds on it in terms of the total variation distance. In particular, \"{O}sterreicher and Vajda~\cite[Theorem~2]{osterreicher2003new} proved that for any $\alpha>0$, probability distributions $P$ and $Q$, we have
\begin{align}\label{eqn:boundsonArimoto}
    \psi_\alpha(D_{\text{TV}}(P||Q))\leq D_{f_\alpha}(P||Q)\leq \psi_\alpha(1)D_{\text{TV}}(P||Q),
\end{align}
where the function $\psi_\alpha:[0,1]\rightarrow \mathcal{R}$ defined by $\psi_\alpha(p)=\frac{\alpha}{\alpha-1}\left(\left(\left(1+p\right)^\alpha+\left(1-p\right)^\alpha\right)^\frac{1}{\alpha}-2^{\frac{1}{\alpha}}\right)$ for $\alpha\in(0,1)\cup(1,\infty)$ is convex and strictly monotone increasing such that $\psi_\alpha(0)=0$ and $\psi_\alpha(1)=\frac{\alpha}{\alpha-1}\left(2-2^\frac{1}{\alpha}\right)$. 

We first prove the `only if' part, i.e., Arimoto divergence is stronger than the total variation distance. Suppose $D_{f_\alpha}(P_n||P)\rightarrow 0$. From the lower bound in \eqref{eqn:boundsonArimoto}, it follows that $\psi_\alpha(D_{\text{TV}}(P_n||P))\leq D_{f_\alpha}(P_n||P)$, for each $n\in\mathbbm{N}$. This implies that $\psi_\alpha(D_{\text{TV}}(P_n||P))\rightarrow 0$. We show below that $\psi_\alpha$ is invertible and $\psi_\alpha^{-1}$ is continuous. Then it would follow that $\psi_\alpha^{-1}\psi_\alpha(D_{\text{TV}}(P_n||P))=D_{\text{TV}}(P_n||P)\rightarrow \psi_\alpha^{-1}(0)=0$ proving that Arimoto divergence is stronger than the total variation distance. It remains to show that $\psi_\alpha$ is invertible and $\psi_\alpha^{-1}$ is continuous. Invertibility follows directly from the fact that $\psi_\alpha$ is strictly monotone increasing function. For the continuity of $\psi_\alpha^{-1}$, it suffices to show that $\psi_\alpha(C)$ is closed for a closed set $C\subseteq [0,1]$. The closed set $C$ is compact since a closed subset of a compact set ($[0,1]$ in this case) is also compact. Note that convexity of $\psi_\alpha$ implies continuity and $\psi_\alpha(C)$ is compact since a continuous function of a compact set is also compact. By Heine-Borel theorem, this gives that $\psi_\alpha(C)$ is closed (and bounded) as desired.

For the `if part', i.e., to prove that the total variation distance is stronger than Arimoto divergence, consider a sequence of probability distributions $(P_n)$ such that $D_{\text{TV}}(P_n||P)\rightarrow 0$. It follows from the upper bound in $\eqref{eqn:boundsonArimoto}$ that $D_{f_\alpha}(P_n||P)\leq D_{\text{TV}}(P_n||P)$, for each $n\in\mathbbm{N}$. This implies that $D_{f_\alpha}(P_n||P)\rightarrow 0$ which completes the proof. 
\section{Proof of Theorem~\ref{prop}}\label{proofoftheorem6}
We know from \cite[Corollary 1]{sypherd2019tunable} that for $\eta\in[0,1]$,
\begin{align*}
    \inf_t \eta\tilde{\ell}_\alpha(t)+(1-\eta)\tilde{\ell}_\alpha(-t)&=\frac{\alpha}{\alpha-1}\left(1-\left(\eta^\alpha+(1-\eta)^\alpha\right)^{\frac{1}{\alpha}}\right).
\end{align*}
This implies that
\begin{multline}\label{eqn:thm31}
\inf_t \frac{\eta}{1-\eta}\tilde{\ell}_\alpha(t)+\tilde{\ell}_\alpha(-t)\\
=\frac{\alpha}{\alpha-1}\left(1+\frac{\eta}{1-\eta}-\left(\left(\frac{\eta}{1-\eta}\right)^\alpha+1\right)^{\frac{1}{\alpha}}\right)
\end{multline}
Now substituting $u$ for $\frac{\eta}{1-\eta}$ and taking negation in \eqref{eqn:thm31}, we get
\begin{multline}
     -\inf_t u\tilde{\ell}_\alpha(t)+\tilde{\ell}_\alpha(-t)=\frac{\alpha}{\alpha-1}\left(\left(u^\alpha+1\right)^{\frac{1}{\alpha}}-(1+u)\right),\\
     \text{for}\ u\geq 0
\end{multline}
giving us \eqref{eqn:thm3stat}.
\bibliographystyle{IEEEtran}
\bibliography{Bibliography}

\begin{thebibliography}{10}
\providecommand{\url}[1]{#1}
\csname url@samestyle\endcsname
\providecommand{\newblock}{\relax}
\providecommand{\bibinfo}[2]{#2}
\providecommand{\BIBentrySTDinterwordspacing}{\spaceskip=0pt\relax}
\providecommand{\BIBentryALTinterwordstretchfactor}{4}
\providecommand{\BIBentryALTinterwordspacing}{\spaceskip=\fontdimen2\font plus
\BIBentryALTinterwordstretchfactor\fontdimen3\font minus
  \fontdimen4\font\relax}
\providecommand{\BIBforeignlanguage}[2]{{%
\expandafter\ifx\csname l@#1\endcsname\relax
\typeout{** WARNING: IEEEtran.bst: No hyphenation pattern has been}%
\typeout{** loaded for the language `#1'. Using the pattern for}%
\typeout{** the default language instead.}%
\else
\language=\csname l@#1\endcsname
\fi
#2}}
\providecommand{\BIBdecl}{\relax}
\BIBdecl

\bibitem{Goodfellow14}
I.~J. Goodfellow, J.~Pouget-Abadie, M.~Mirza, B.~Xu, D.~Warde-Farley, S.~Ozair,
  A.~Courville, and Y.~Bengio, ``Generative adversarial nets,'' in
  \emph{Proceedings of the 27th International Conference on Neural Information
  Processing Systems - Volume 2}, 2014, p. 2672–2680.

\bibitem{lim2017geometric}
J.~H. Lim and J.~C. Ye, ``Geometric {GAN},'' \emph{arXiv preprint
  arXiv:1705.02894}, 2017.

\bibitem{cai2020utilizing}
L.~Cai, Y.~Chen, N.~Cai, W.~Cheng, and H.~Wang, ``Utilizing amari-alpha
  divergence to stabilize the training of generative adversarial networks,''
  \emph{Entropy}, vol.~22, no.~4, p. 410, 2020.

\bibitem{NowozinCT16}
S.~Nowozin, B.~Cseke, and R.~Tomioka, ``$f$-{GAN}: Training generative neural
  samplers using variational divergence minimization,'' in \emph{Proceedings of
  the 30th International Conference on Neural Information Processing Systems},
  2016, p. 271–279.

\bibitem{ArjovskyCB17}
M.~Arjovsky, S.~Chintala, and L.~Bottou, ``{W}asserstein generative adversarial
  networks,'' in \emph{Proceedings of the 34th International Conference on
  Machine Learning}, vol.~70, 2017, pp. 214--223.

\bibitem{liang2018well}
T.~Liang, ``How well generative adversarial networks learn distributions,''
  \emph{arXiv preprint arXiv:1811.03179}, 2018.

\bibitem{huszar2015not}
F.~Husz{\'a}r, ``How (not) to train your generative model: Scheduled sampling,
  likelihood, adversary?'' \emph{arXiv preprint arXiv:1511.05101}, 2015.

\bibitem{metz2016unrolled}
L.~Metz, B.~Poole, D.~Pfau, and J.~Sohl-Dickstein, ``Unrolled generative
  adversarial networks,'' \emph{arXiv preprint arXiv:1611.02163}, 2016.

\bibitem{salimans2016improved}
T.~Salimans, I.~Goodfellow, W.~Zaremba, V.~Cheung, A.~Radford, and X.~Chen,
  ``Improved techniques for training {GAN}s,'' \emph{arXiv preprint
  arXiv:1606.03498}, 2016.

\bibitem{arjovsky2017towards}
M.~Arjovsky and L.~Bottou, ``Towards principled methods for training generative
  adversarial networks,'' \emph{arXiv preprint arXiv:1701.04862}, 2017.

\bibitem{GulrajaniAADC17}
I.~Gulrajani, F.~Ahmed, M.~Arjovsky, V.~Dumoulin, and A.~C. Courville,
  ``Improved training of {W}asserstein {GAN}s,'' in \emph{Advances in Neural
  Information Processing Systems}, vol.~30, 2017.

\bibitem{sypherd2019tunable}
T.~Sypherd, M.~Diaz, L.~Sankar, and P.~Kairouz, ``A tunable loss function for
  binary classification,'' in \emph{IEEE International Symposium on Information
  Theory}, 2019, pp. 2479--2483.

\bibitem{SypherdDSD20}
T.~Sypherd, M.~Diaz, L.~Sankar, and G.~Dasarathy, ``On the $\alpha$-loss
  landscape in the logistic model,'' in \emph{IEEE International Symposium on
  Information Theory}, 2020, pp. 2700--2705.

\bibitem{FREUND1997119}
Y.~Freund and R.~E. Schapire, ``A decision-theoretic generalization of on-line
  learning and an application to boosting,'' \emph{Journal of Computer and
  System Sciences}, vol.~55, no.~1, pp. 119 -- 139, 1997.

\bibitem{MerhavF1998}
N.~{Merhav} and M.~{Feder}, ``Universal prediction,'' \emph{IEEE Transactions
  on Information Theory}, vol.~44, no.~6, pp. 2124--2147, 1998.

\bibitem{CourtadeW11}
T.~A. {Courtade} and R.~D. {Wesel}, ``Multiterminal source coding with an
  entropy-based distortion measure,'' in \emph{IEEE International Symposium on
  Information Theory}, 2011, pp. 2040--2044.

\bibitem{NguyenWJ09}
X.~Nguyen, M.~J. Wainwright, and M.~I. Jordan, ``On surrogate loss functions
  and f-divergences,'' \emph{The Annals of Statistics}, vol.~37, no.~2, pp.
  876--904, 2009.

\bibitem{BartlettJM06}
P.~L. Bartlett, M.~I. Jordan, and J.~D. Mcauliffe, ``Convexity, classification,
  and risk bounds,'' \emph{Journal of the American Statistical Association},
  vol. 101, no. 473, pp. 138--156, 2006.

\bibitem{LieseV06}
F.~Liese and I.~Vajda, ``On divergences and informations in statistics and
  information theory,'' \emph{IEEE Transactions on Information Theory},
  vol.~52, no.~10, pp. 4394--4412, 2006.

\bibitem{osterreicher2003new}
F.~{\"O}sterreicher and I.~Vajda, ``A new class of metric divergences on
  probability spaces and its applicability in statistics,'' \emph{Annals of the
  Institute of Statistical Mathematics}, vol.~55, no.~3, pp. 639--653, 2003.

\bibitem{sypherd2021journal}
T.~Sypherd, M.~Diaz, J.~K. Cava, G.~Dasarathy, P.~Kairouz, and L.~Sankar, ``A
  tunable loss function for robust classification: Calibration, landscape, and
  generalization,'' \emph{arXiv preprint arXiv:1906.02314}, 2019.

\bibitem{osterreicher1996class}
F.~{\"O}sterreicher, ``On a class of perimeter-type distances of probability
  distributions,'' \emph{Kybernetika}, vol.~32, no.~4, pp. 389--393, 1996.

\bibitem{liu2017approximation}
S.~Liu, O.~Bousquet, and K.~Chaudhuri, ``Approximation and convergence
  properties of generative adversarial learning,'' \emph{Advances in Neural
  Information Processing Systems}, vol.~30, 2017.

\bibitem{arimoto1971information}
S.~Arimoto, ``Information-theoretical considerations on estimation problems,''
  \emph{Information and control}, vol.~19, no.~3, pp. 181--194, 1971.

\bibitem{liao2018tunable}
J.~Liao, O.~Kosut, L.~Sankar, and F.~P. Calmon, ``A tunable measure for
  information leakage,'' in \emph{2018 IEEE International Symposium on
  Information Theory (ISIT)}.\hskip 1em plus 0.5em minus 0.4em\relax IEEE,
  2018, pp. 701--705.

\bibitem{Lin91}
J.~Lin, ``Divergence measures based on the shannon entropy,'' \emph{IEEE
  Transactions on Information Theory}, vol.~37, no.~1, pp. 145--151, 1991.

\bibitem{measures_renyi1961}
A.~R{\'e}nyi, ``On measures of entropy and information,'' in \emph{Proceedings
  of the Fourth Berkeley Symposium on Mathematical Statistics and Probability},
  1961, pp. 547--561.

\bibitem{Csiszar67}
I.~Csisz\'{a}r, ``Information-type measures of difference of probability
  distributions and indirect observation,'' \emph{Studia Scientiarum
  Mathematicarum Hungarica}, vol.~2, pp. 229--318, 1967.

\bibitem{Alis66}
S.~M. Ali and S.~D. Silvey, ``A general class of coefficients of divergence of
  one distribution from another,'' \emph{Journal of the Royal Statistical
  Society. Series B (Methodological)}, vol.~28, no.~1, pp. 131--142, 1966.

\bibitem{NguyenWJ10}
X.~Nguyen, M.~J. Wainwright, and M.~I. Jordan, ``Estimating divergence
  functionals and the likelihood ratio by convex risk minimization,''
  \emph{IEEE Transactions on Information Theory}, vol.~56, no.~11, pp.
  5847--5861, 2010.

\bibitem{villani2008optimal}
C.~Villani, \emph{Optimal transport: old and new}.\hskip 1em plus 0.5em minus
  0.4em\relax Springer Science \& Business Media, 2008, vol. 338.

\bibitem{wiatrak2019stabilizing}
M.~Wiatrak, S.~V. Albrecht, and A.~Nystrom, ``Stabilizing generative
  adversarial networks: A survey,'' \emph{arXiv preprint arXiv:1910.00927},
  2019.

\bibitem{AroraGLMZ17}
S.~Arora, R.~Ge, Y.~Liang, T.~Ma, and Y.~Zhang, ``Generalization and
  equilibrium in generative adversarial nets ({GAN}s),'' in \emph{Proceedings
  of the 34th International Conference on Machine Learning}, vol.~70, 2017, pp.
  224--232.

\bibitem{reid2010composite}
M.~D. Reid and R.~C. Williamson, ``Composite binary losses,'' \emph{The Journal
  of Machine Learning Research}, vol.~11, pp. 2387--2422, 2010.

\bibitem{pmlr-v70-mroueh17a}
Y.~Mroueh, T.~Sercu, and V.~Goel, ``{M}c{G}an: Mean and covariance feature
  matching {GAN},'' in \emph{Proceedings of the 34th International Conference
  on Machine Learning}, vol.~70, 2017, pp. 2527--2535.

\bibitem{MrouehS17}
Y.~Mroueh and T.~Sercu, ``Fisher {GAN},'' in \emph{Advances in Neural
  Information Processing Systems}, vol.~30, 2017.

\bibitem{mroueh2017sobolev}
Y.~Mroueh, C.-L. Li, T.~Sercu, A.~Raj, and Y.~Cheng, ``Sobolev {GAN},''
  \emph{arXiv preprint arXiv:1711.04894}, 2017.

\bibitem{Amari1985}
S.-i. Amari, \emph{$\alpha$-Divergence and $\alpha$-Projection in Statistical
  Manifold}.\hskip 1em plus 0.5em minus 0.4em\relax New York, NY: Springer New
  York, 1985, pp. 66--103.

\bibitem{cerone2004bound}
P.~Cerone, S.~S. Dragomir, and F.~{\"O}sterreicher, ``Bound on extended $ f
  $-divergences for a variety of classes,'' \emph{Kybernetika}, vol.~40, no.~6,
  pp. 745--756, 2004.

\bibitem{vajda2009metric}
I.~Vajda, ``On metric divergences of probability measures,''
  \emph{Kybernetika}, vol.~45, no.~6, pp. 885--900, 2009.

\bibitem{dziugaite2015training}
G.~K. Dziugaite, D.~M. Roy, and Z.~Ghahramani, ``Training generative neural
  networks via maximum mean discrepancy optimization,'' \emph{arXiv preprint
  arXiv:1505.03906}, 2015.

\end{thebibliography}
\end{document}